\documentclass{article}
\usepackage{spconf,amsmath,graphicx}
\usepackage[caption=false]{subfig}
\usepackage{xcolor,booktabs,amsfonts,rotating,tikz}
\usepackage{algorithm}
\usepackage[noend]{algorithmic}
\usepackage{amsthm}
\usepackage{amssymb}
\usepackage{hyperref}
\newcommand{\rottriangle}{
	\begin{sideways}\begin{sideways}      \begin{sideways}$\triangle$\end{sideways}\end{sideways}\end{sideways}}
\DeclareMathOperator*{\argmax}{arg\,max}

\title{Energy Regularized RNNs for Solving Non-Stationary Bandit Problems}
\name{Michael Rotman\quad ~~\quad ~~\quad Lior Wolf}
\address{School of Computer Science, Tel Aviv University}
\begin{document}
	%
	\maketitle
	\begin{abstract}
		We consider a Multi-Armed Bandit problem in which the rewards are non-stationary and are dependent on past actions and potentially on past contexts. At the heart of our method, we employ a recurrent neural network, which models these sequences. 
		In order to balance between exploration and exploitation, we present an energy minimization term that prevents the neural network from becoming too confident in support of a certain action. This term provably limits the gap between the maximal and minimal probabilities assigned by the network. In a diverse set of experiments, we demonstrate that our method is at least as effective as methods suggested to solve the sub-problem of Rotting Bandits, and can solve intuitive extensions of various benchmark problems. We share our implementation at
		\href{https://github.com/rotmanmi/Energy-Regularized-RNN}{https://github.com/rotmanmi/Energy-Regularized-RNN}.
	\end{abstract}

	\section{Introduction}
	
	The clich\'{e} ``insanity is doing the same thing over and over again and expecting different results'' is wrong, since our experience tells us that we cannot exploit the same action repeatedly and expect to enjoy the same outcome. While in the conventional Multi-Armed Bandits (MAB) setting~\cite{thompson1933likelihood}, the reward distribution of each arm is assumed to be stationary, real-world scenarios, such as online advertising and content recommendation have led to more general settings. For example, in Rotting Bandits~\cite{levine2017rotting}, the reward decays in accordance with the number of times that an arm has been pulled. Rotting Bandits, however, do not address the cases in which the reward is dependent on the complete history of the arm pulling actions, which also takes into account the pulling of other arms, as well as the order of the actions.

	To provide a generic treatment, we employ a Recurrent Neural Network (RNN) in order to model the non-stationary reward distributions. For tasks in which the decision at each round benefits from a set of observations, the time-dependent context is fed as the input to the RNN. Thus, our approach naturally extends a non-stationary form of Contextual Bandits~\cite{langford2007epoch}. Unlike the conventional Contextual Bandit case, the optimal action predictions are  not only conditioned on the most recent context, but also on all previous contexts.
	
	The learned policy selects the action based on a Softmax that is applied to a set of logits. Similar to what is observed in other contexts of Reinforcement Learning~\cite{tijsma2016comparing} and also in Continual Learning~\cite{serra2018overcoming}, this may lead to an overconfident network. Such a confident decision increases exploitation at the expense of exploration. In order to overcome this, we add a novel regularization term that reduces the Boltzmann energy of the logits. This term is shown to directly lead to an upper bound on the ratio between the maximal probability assigned by the model to an arm and the minimal probability.
	
	Our experiments are conducted along multiple axes.  First, we show that our method performs well on common stationary problems. Second, that it outperforms multiple leading baselines on non-stationary extensions of these problems. Third, that our generic method outperforms, in terms of convergence time, the leading methods for Rotting Bandits.  
	In a set of ablation experiments, we also demonstrate that the energy regularization method we advocate for, outperforms other exploration enhancing alternatives.

	\section{Related Work}
	\noindent{\bf Non-Stationary Multi-Arm Bandits\quad} Non-stationarity arises when considering time-dependent priors~\cite{besbes2014stochastic,garivier2011upper,wei2016tracking,allesiardo2017non,chen2021decision,cheung2022hedging} or when only a subset of the bandit's arms is available at a given time such as in the Sleeping Bandit~\cite{kanade2009sleeping,li2019combinatorial}. However, it also appears when there is no explicit time-dependency. For instance, when the reward of a given arm decays with each pull such as in the Rotting Bandit~\cite{levine2017rotting,seznec2019rotting} or when the rewards are received in delay~\cite{liu2019multi}. Contextual Bandits~\cite{li2010contextual,luo2018efficient,auer2019achieving} aim to relate between a given context and the bandit problem. It has been successfully applied for recommendation tasks~\cite{li2010contextual} and for a non-linear reward structure~\cite{krishnamurthy2016contextual}.

	\noindent{\bf Exploration Strategies In Neural Networks\quad}
	MAB problems require the balancing of exploration and exploitation.  The $\epsilon$-greedy method, explores with probability $\epsilon$~\cite{thrun1992efficient}. The softmax regularization smooths the probability distribution using a high temperature~\cite{tijsma2016comparing}. The Upper Confidence Bound was also used to draw from a smooth distribution~\cite{zhou2020neural}. 
	
	\section{Method}
	
	Let A be a set of $n$ possible actions, $\left\{ a_i \right\}_{i=1}^n$. At each time step $1\leq t \leq T$, a context, $c_t\in\mathbb{R}^{m}$, is presented to the observer and a reward, $r_{i,t}$, is assigned to each action. This framework applies both to cases in which the context is meaningful $m\!>\!0$ and in cases in which the context is void ($m\!=\!0$).
	
	In the bandit setting, only the reward $r_{i,t}$ associated with the action selected at time $t$ is presented to the observer. Let $\left\{a_{i_t}\right\}_{t=1}^T$ be the set of actions selected by the observer. The objective of the learner is to minimize the regret,
	$R = \sum_t \max_k \left\{r_{k,t}\right\}  - r_{i_t,t}$.  
	In order to select an action at time step $t$, a $l$-layer stacked GRU~\cite{cho2014learning}, $f$, with a time-dependent hidden state, $h_t$, of size $d$ is employed. The GRU receives as input the previous hidden state, $h_{t-1}$, together with the context vector, $c_t$, and updates the hidden state $h_t = f\left(c_t,h_{t-1} \right)$. 
	Next, an affine layer, $o$, acts on the hidden state, $h_t$, obtaining    $z_i = o\left( h_t \right)$, which is fed 
	to a Softmax layer, providing the probability $p_i$ of selecting action $a_i$.
	
	Under this setting, only partial information is available, as only the reward for action $i_t$ is revealed. Since the regret $R$ 
	is non-differentiable, we employ the  REINFORCE~\cite{williams1992simple} algorithm in order to minimize it. The loss function, $\mathcal{L}_R$, at time $t$ is $\mathcal{L}_R = -\left(r_{i_t,t} - \bar{r}_t \right) \log p_{i_t}$, 
	where $\bar{r}_t$ is the mean of the obtained rewards up to time $t$.
	
	Since the observer is unaware of the outcomes of other actions at time $t$, selecting the next action requires a subtle balance between exploration and exploitation. Whereas in neural networks, exploitation occurs naturally, we encourage exploration in three manners. First, we apply a dropout with probability $p_{dropout}$ after each GRU layer, thus adding uncertainty to the predicted actions. Second, we sample the next action by considering the probability, $p_i$, of the action and not by choosing the one with maximal $p_i$. Third, we add an energy conservation term to the loss function, thus regulating the values of the network's parameters, so the output of the classification layer cannot become biased.
	
	The Softmax function can be interpreted as the Boltzmann distribution, where each neuron $z_i$ in the classification layer carries an energy of $E_i = -z_i$.  The mean energy of a system of neurons is $\left\langle E \right\rangle=\sum_{i=1}^n  p_i E_i$. 
	In order to conserve energy, the following regularization term is added to the loss function,
	\begin{equation}
		\mathcal{L}_{EC} = \left\langle E \right\rangle^2 = \left( \sum_{i=1}^n p_{i} E_i  \right)^2  = \left( \sum_{i=1}^n p_{i} z_i  \right)^2   \,. \label{eq:const}
	\end{equation}
	Since $\mathcal{L}_{EC}$ is quadratic, it has a minimum only when $ \left\langle E \right\rangle = 0$. 
	This ensures that if an activation, $z_i$, is drifting to a high value (which results in a high probability, $p_i$, for selecting action $a_i$), $\mathcal{L}_{EC}$ will lead to a compensation by increasing the values of the other neurons $z_j$ as well.
	Specifically, the following bound holds.
	\newtheorem{thm}{Theorem}
	\begin{thm}[]
		Let $z_i$ be a set of $n$ logits that are passed to a Softmax to obtain probabilities $p_i$ and are sufficiently regularized by $\mathcal{L}_{EC}$. Denote by $p_{\max}$ and $p_{\min}$ the maximal and minimal values of $p_i$. Then $ \frac{p_{\max}}{p_{\min}} \leq e^{W\left(\frac{n-1}{e}\right) + 1}$, 
		where $W$ is the Lambert-W function defined as $W\left(x\right) e^{W\left(x\right)}  = x$.
	\end{thm}
	\begin{proof}
		Denote by $z_{\max}$ and $z_{\min}$ the maximal and minimal $z_i$ values and by $i_{\max}$ the index of $z_{\max}$. It follows from $\sum_{i=1}^n p_{i} z_i =0$ 
		that
		\begin{equation}
			e^{z_{\max}}z_{\max} = -\sum_{i\neq i_{\max}} e^{z_i}z_i \leq -\left( n-1 \right) e^{z_{\min}}z_{\min} \,. \label{eq:zmin}
		\end{equation}
		The minimum of the function $y(x) = x e^x$ is  at $y=-\frac{1}{e}$ which means $e^{z_{\min}}z_{\min} \geq -\frac{1}{e}$. Plugging this to \eqref{eq:zmin}, $e^{z_{\max}}z_{\max} \leq  \frac{n-1}{e}$.
		Since $z_{\max}$ must be positive to satisfy  $\left\langle E \right\rangle = 0$, and because $xe^x$ is monotonic for $x>0$, $z_{\max} \leq W\left( \frac{n-1}{e}\right)$.  
		The ratio between the probabilities $p_{\max}$ and $p_{\min}$ is 
		\begin{equation}
			\nonumber
			\frac{p_{\max}}{p_{\min}} = \frac{e^{z_{\max}}}{e^{z_{\min}}} \leq \frac{e^{W\left(\frac{n-1}{e}\right)}}{e^{W\left(-\frac{1}{e}\right)}} = e^{W\left(\frac{n-1}{e}\right) + 1}\,. \eqno\qed
		\end{equation}
		\renewcommand{\qedsymbol}{}
	\end{proof}
	
	For small values of $n$, the bound grows almost linearly. 
	It limits the highest probability gap between the arms, thus allowing less-visited actions to be explored. 
	
	The loss terms are combined to obtain the optimization goal $\mathcal L= \mathcal L_R +\alpha_{EC} \mathcal{L}_{EC}$, 
	where $\alpha_{{EC}}$ is a regularization hyperparameter. The entire procedure procedure follows Alg.~\ref{alg}. 
	
	\renewcommand\algorithmiccomment[1]{%
		\hfill\rottriangle\ {#1}%
	}
	\begin{algorithm}[t]
		\caption{The Recurrent Bandit Algorithm}
		\label{alg}
		\begin{algorithmic}
			\renewcommand{\algorithmicrequire}{\textbf{Input:}}
			\renewcommand{\algorithmicensure}{\textbf{Output:}}
			\STATE $h_0 = 0$ \\
			\FOR {$t = 1$ to $T$}
			\STATE $h_t = f\left(c_t,h_{t-1}\right)$ \\
			\STATE $p_i = \text{Softmax}(o(h_t))$\COMMENT{Action probabilities}\\
			\STATE Select action $a_{i_t}$ by sampling $p_i$ \\
			\STATE Observe the reward $r_{i_t,t}$ \\ 
			\STATE $\mathcal{L} =  \mathcal{L}_R + \alpha_{EC}\mathcal{L}_{EC}$ \COMMENT{Combined loss term}\\
			\STATE Perform back-propagation minimizing $\mathcal{L}$
			\ENDFOR
		\end{algorithmic} 
	\end{algorithm}
	
	\begin{figure*}[t]
		\centering%
		\subfloat[]{%
			\centering%
			\includegraphics[width=0.25\linewidth]{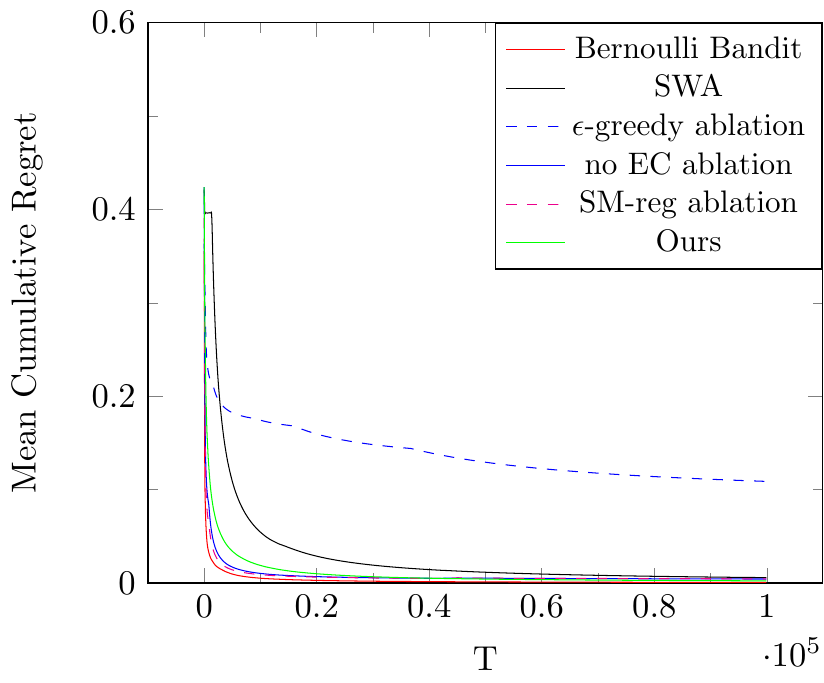}%
			\label{subfig:basic1}%
		}%
		\subfloat[]{%
			\centering
			\includegraphics[width=0.25\linewidth]{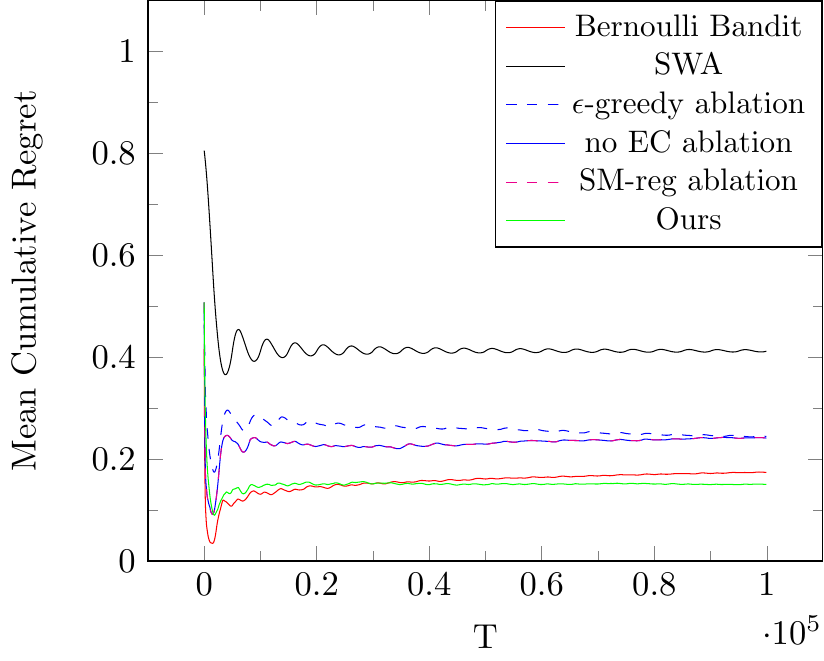}%
			\label{subfig:basic2}%
		}%
		\subfloat[]{%
			\centering%
			\includegraphics[width=0.25\linewidth]{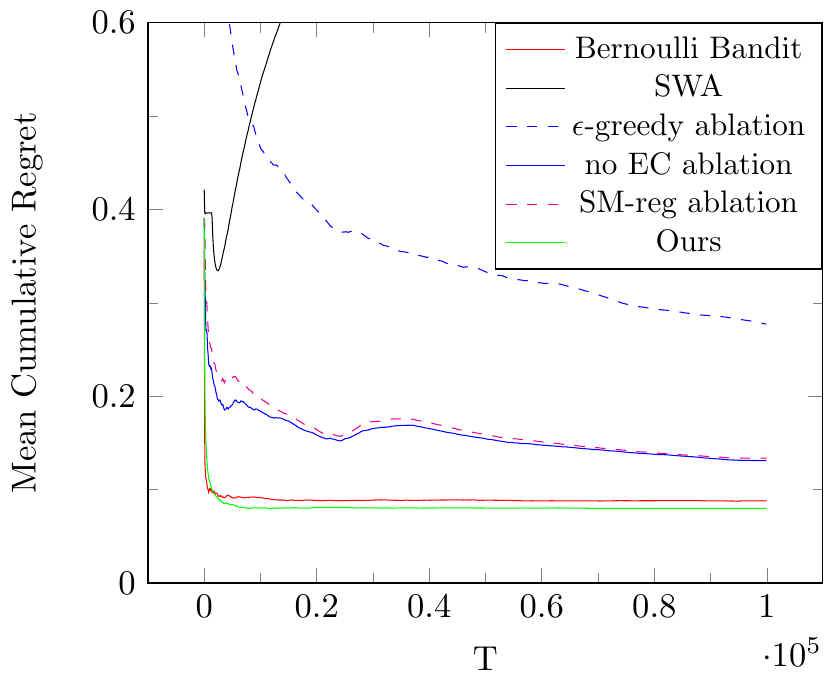}%
			\label{subfig:basic3}%
		}%
		\subfloat[]{%
			\centering%
			\includegraphics[width=0.25\linewidth]{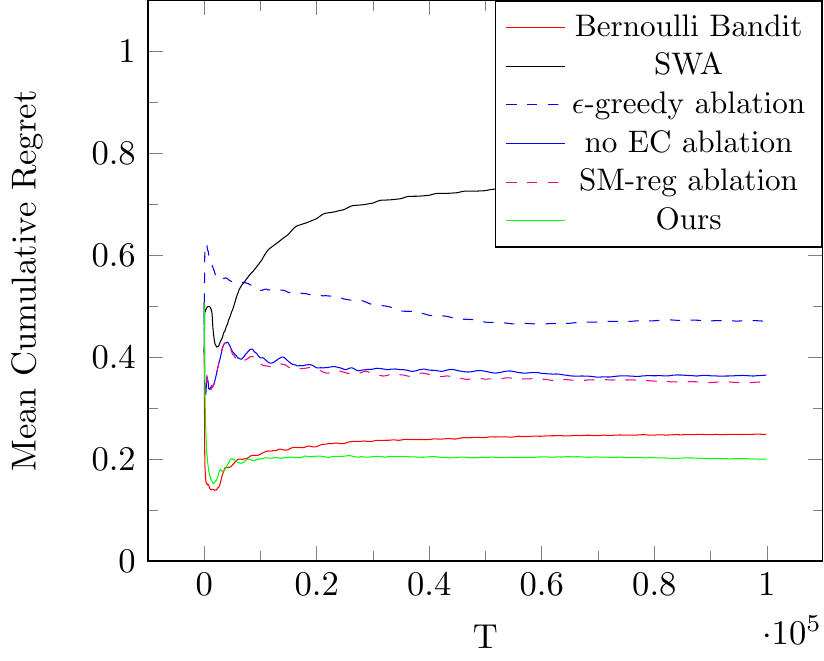} %
			\label{subfig:basic4}%
		}%
		\vspace{-2mm}
		\caption{%
			(a) A MAB with $10$ actions. (b) A sinusidial MAB with $10$ actions. (c) A MAB with $10$ actions, where the number of consecutive pulls of a handle is limited. (d) Sinusidial MAB with $10$ actions, limiting the number of consecutive pulls. }
		\label{fig:mabs}
	\end{figure*}
	
	\begin{figure}[t]
		\centering%
		\subfloat[]{%
			\centering%
			\includegraphics[width=0.39\linewidth]{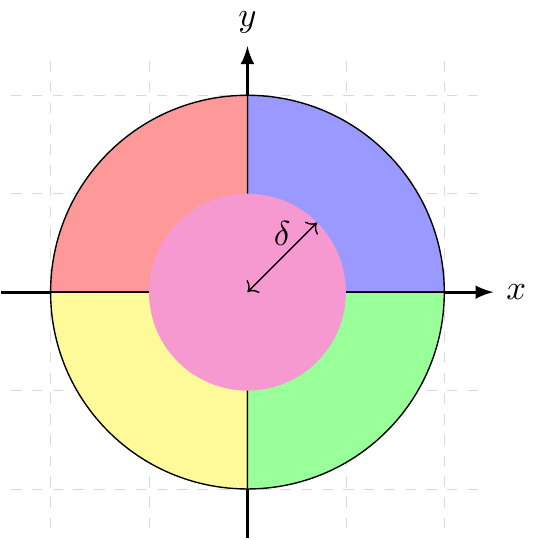}%
			\label{subfig:wheel1}%
		}%
		\hfill%
		\subfloat[]{%
			\centering
			\includegraphics[width=0.39\linewidth]{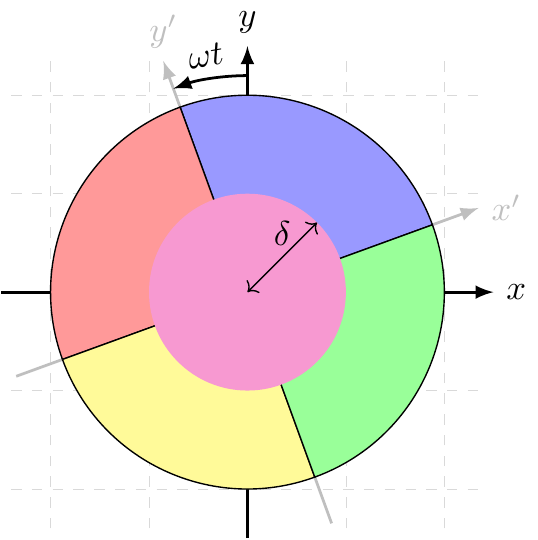}%
			\label{subfig:wheel2}%
		}%
		\caption{Conventional (a) and Rotating (b) Wheel Bandits. Each colored region describes a different reward distribution.}
		\label{fig:wheelsetup}
	\end{figure}

	\begin{figure}
		\centering%
		\subfloat[]{%
			\includegraphics[width=0.50\linewidth]{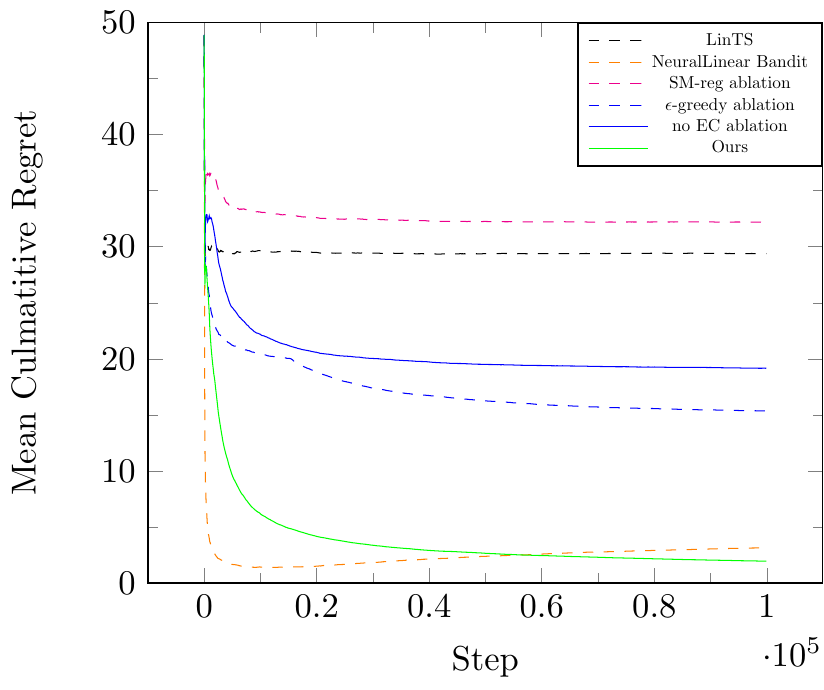}%
			\label{subfig:reswheel1}%
		}%
		\hfill
		\subfloat[]{%
			\includegraphics[width=0.5\linewidth]{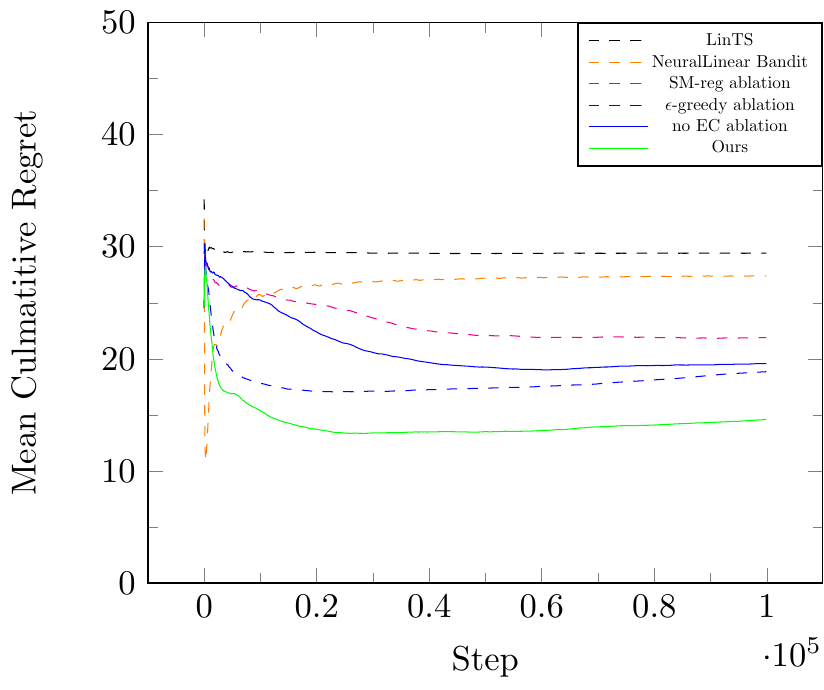}%
			\label{subfig:reswheel2}%
		}%
		\caption{(a) Wheel Bandit. (b) Rotating Wheel Bandit, $T_{\text{period}}=2000$. LinTS by \protect\cite{cortes2018adapting}, NeuralLinear by \protect\cite{riquelme2018deep}.}
		\label{fig:wheel}
	\end{figure}
	
	\begin{figure*}[t]
		\centering
		\begin{minipage}[t]{0.31\linewidth}
			\includegraphics[
			clip,width=0.90\linewidth]{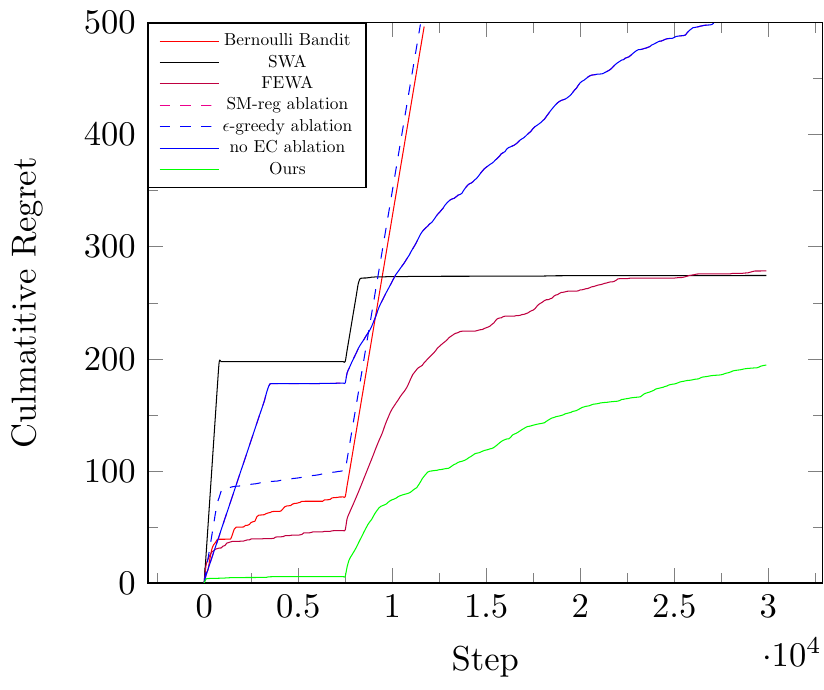}
			\vspace{-2mm}
			\caption{Cumulative regret for Rotting Bandits.}
			\label{fig:rotting}
		\end{minipage}%
		\hfill
		\begin{minipage}[t]{0.31\linewidth}
			\centering%
			\includegraphics[width=0.9\linewidth]{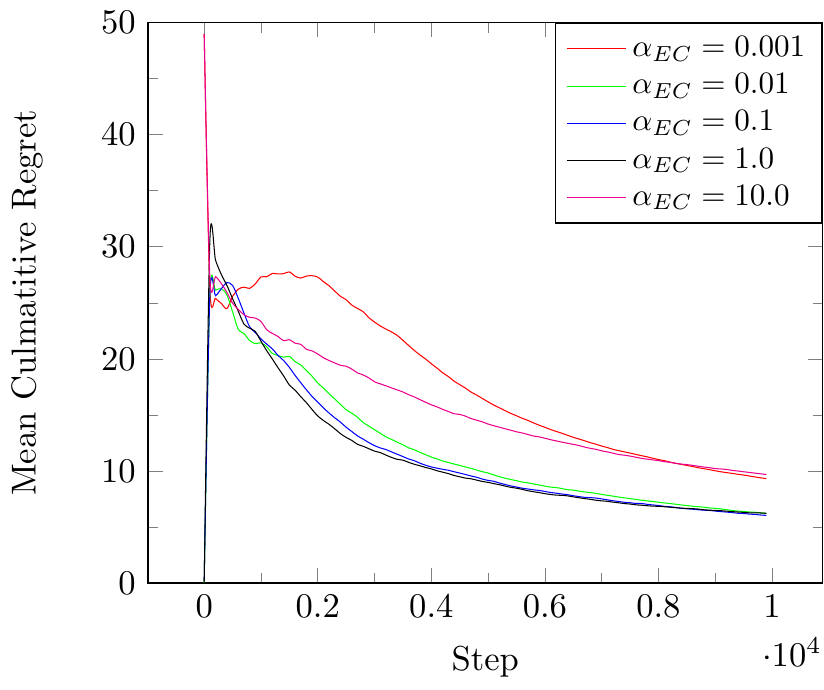}%
			\vspace{-2mm}
			\caption{Cumulative Regret on the stationary Wheel varying $\alpha_{EC}$.}
			\label{fig:senswheel1}
		\end{minipage}%
		\hfill
		\begin{minipage}[t]{0.31\linewidth}
			\centering%
			\includegraphics[width=0.9\linewidth]{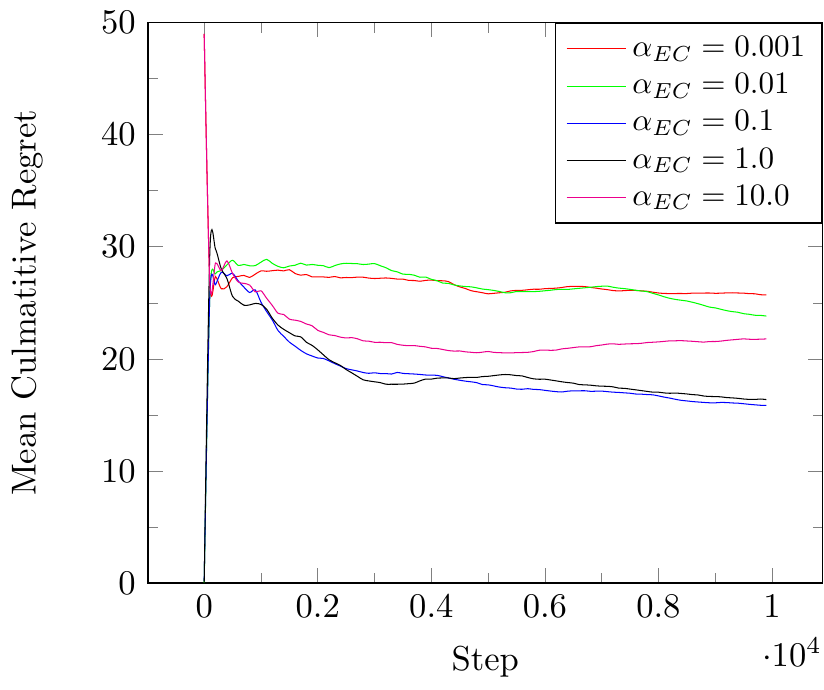}%
			\vspace{-2mm}
			\caption{ Cumulative Regret on the Rotating Wheel dataset for different  $\alpha_{EC}$.}
			\label{fig:senswheel2}   \end{minipage}%
	\end{figure*}
	\section{Experiments}
	\label{sec:experiments}
	
	We apply our method on a variety of different multi-arm bandit benchmarks. For all the tasks, except for the Wheel Bandit, 
	the context vector, $c_t$, that is fed to the Recurrent Bandit is null. 
	We optimize our method using RMSprop~\cite{bengio2015rmsprop} with a learning rate of $0.001$. The network $f$ for all the experiments is a 2-layer stacked GRU, each with a hidden size of $d_h = 128$. The dropout used for all the experiments is $p_{\text{dropout}}=0.1$, except for the Rotting Bandits, where it is set to $p_{\text{dropout}}=0$, since it does not require much exploration due to its nature. The regularization parameter is fixed at a value of $\alpha_{EC}=0.1$ for all the benchmarks. 
	
	An ablation study is performed to verify the benefits of using the novel loss term $\mathcal{L}_{EC}$. $\epsilon$-greedy ablation uses the same GRU architecture, however at each time step chooses a random action with a probability of $\epsilon=0.01$, and the most probable action with probability $1-\epsilon$. The Softmax-reg ablation uses Softmax regularization to smooth the action selection using a temperature of $T=2$. The no EC ablation is obtained by setting $\alpha_{EC}=0$. 
	
	\noindent{\bf Bernoulli Multi-Arm Bandit\quad}
	A set of $n$ handles or actions, are assigned with a uniform prior $\{\bar{\theta}_k\}_{k=1}^n$. At each time step, the user has to pick one action. each of the actions carries a reward, $r_k\in \left\{0,1 \right\}$, sampled from the Bernoulli's distribution, $    r_k \sim Bernoulli\left( \bar{\theta}_k \right)$.
	We compare our method to the Bernoulli Bandit~\cite{chapelle2011empirical}, which is based on Thompson-Sampling, to SWA~\cite{levine2017rotting} that was designed for the Rotting Bandit problem, and to the three ablation variants of our method. Fig.~\ref{subfig:basic1} shows the mean accumulated regret for $100$ simulations for each of the methods. As can be seen, all methods are able to reduce the mean accumulated regret over time. However, SWA takes longer to accomplish this. Our method works best in this case if regularization is turned off, and its regret over the Bernoulli Bandit is not large.

	\noindent{\bf Time Dependent Multi-Arm Bandit\quad}
	One way to create a non-stationary bandit problem is to add time-dependent priors to the Bernoulli Multi-Arm Bandit task. To accomplish this, time dependency is inserted to the priors $\bar{\theta}_k$ using the mirrored sine function      $\theta_k \left( t \right) =  \bar {\theta}_k \left\vert \sin \left( \omega t + \phi_k \right) \right\vert$, 
	where $\omega = \frac{2\pi}{T_{\text{period}}}$ and $\phi_k = \frac{2 \pi k}{n}$ denotes the phase between the arms, so that the arm holding the highest reward switches its place through time. For this experiment, there are $n=10$ available arms and a periodicity of ${T}_{\text{period}} = 10000$. The mean accumulated regret is presented in Fig.~\ref{subfig:basic2}. As can be seen, the drifting of the rewards over time makes it hard for both the Bernoulli Bandit, and the baselines without the Energy Conservation to keep a low regret. The SWA algorithm fails to solve this task, since it was designed specifically for decaying rewards. The variant of our methods that are meant to evaluate other forms of encouraging exploration also fail (this is also true for all other values we tested for their regularization parameter).
	
	\noindent{\bf Time Dependent Multi-Arm Bandit with Correlative Arms\quad}
	In order to further investigate the non-stationary behavior, we limit the number of consecutive selections of a single arm to $10$. After an arm has been selected for $10$ times in a row, its reward is zeroed unless another arm has been selected.  The results of applying this setting to the Bernoulli MAB problem are shown in Fig.~\ref{subfig:basic3}. In this case, only the Bernoulli Bandit and our method maintain a low regret. Under this setting, SWA is not able to solve this task. 
	
	We then combine the two MAB variants and zero the reward after $10$ consecutive pulls for the Time Dependent MAB. The results are shown in Fig.~\ref{subfig:basic4}. Our method is the only one to maintain a low mean cumulative regret. The Bernoulli Bandit regret increases with time, and the ablation variants are not able to converge to low solution.

	\noindent{\bf Wheel Bandit\quad}
	The Wheel Bandit by \cite{riquelme2018deep} has been introduced to examine the expressiveness of contextual bandits in non-linear tasks. The context  $c_t = \left(x_t,y_t \right) \in \mathbb{R}^2$ in this task is uniformly sampled in the unit circle ($R=1$). In this settings, five actions, $a_i$, are available. The first action, $a_1$, always grants a reward, $r\sim \mathcal{N} \left(\mu_1, \sigma \right)$ independent of the context, $c_t$. Inside the region $d = \sqrt{x^2 + y^2} \leq \delta$, the other four available actions, $a_{2-5}$, grant a reward of $r\sim \mathcal{N} \left(\mu_2, \sigma \right)$, for $\mu_2 < \mu_1$. For the region $d > \delta$, depending on the signs on $x$ and $y$ (to which quarter of $\mathbb{R}^2$ they belong to) one of the four arms grants a reward of $r\sim \mathcal{N} \left(\mu_3, \sigma \right)$, where $\mu_3 \gg \mu_1$, whereas the remaining arms grant a reward of  $r\sim \mathcal{N} \left(\mu_2, \sigma \right)$. We use the same settings as the original Wheel bandit, $\mu_1 = 1.2,\mu_2=1.0,\mu_3 = 50,\sigma=0.01$, 
	with an exploration motivating parameter $\delta = 0.5$. 
	
	To make this problem non-stationary, a time-dependent rotation is applied by making the actions' reward distributions depend on the sign of $x'$ and $y'$,
	\begin{equation}
		\left(\begin{array}{c} x' \\ y' \end{array} \right) = \left(\begin{array}{cc} \cos\left(\omega t \right) & \sin\left(\omega t \right) \\ -\sin\left(\omega t \right) & \cos\left(\omega t \right) \end{array} \right)\left(\begin{array}{c} x \\ y \end{array} \right) \,.
	\end{equation}
	Both setups are depicted in Fig.~\ref{fig:wheelsetup}.
	We run both setups for $T=10000$, and for the Rotating Wheel Bandit, we use an $\omega=\frac{2\pi}{2000}$ ($5$ rotations). As can be seen in Fig.~\ref{fig:wheel}(a), our method leads over all baseline methods for the conventional problem, outperforming the previous SOTA NeuralLinear~\cite{riquelme2018deep}. when time-dependency is introduced, NeuralLinear quickly diverges, whereas our method is still able to learn the dynamics, as can be seen in Fig.~\ref{fig:wheel}(b).

	\noindent{\bf Rotting Bandit\quad } The Rotting Bandit was introduced  to handle cases where the Bandit's arms are not stationary, but rather decay over time with each pull~\cite{levine2017rotting}. We follow the original scenario, by considering two available actions, $a_1$ and $a_2$, for $T=30000$ time steps, with rewards, $r_i$, sampled from,
	\begin{eqnarray}
		r_1 &\sim& \mathcal{N}\left(\mu_1=0.5,\sigma^2=0.2 \right) \\
		r_2 &\sim& \left\{ \begin{array}{lc} \mathcal{N}\left(\mu_2=1,\sigma^2=0.2 \right) & n_2 < 7500 \\ 
			\mathcal{N}\left(\mu_2=0.4,\sigma^2=0.2 \right) & else \end{array}\right. \,,
	\end{eqnarray}
	where $n_i$ is the number of times action $a_i$ has been selected. For this experiment, the regret presented is the Policy Regret~\cite{arora2012online}, since the optimal policy~\cite{levine2017rotting} for this task is,
	$a_{i_t} = \argmax_i\left\{ r_i \left( N_{i,t} + 1\right)\right\}$, 
	where $N_{i,t}$ is the number of times action $a_i$ was selected up to time $t$, and $r_i\left(n_i\right)$ is the reward given by action $a_i$.
	
	We compare our method to SWA~\cite{levine2017rotting},  FEWA~\cite{seznec2019rotting} (two Rotting Bandit methods), and to the Vanilla Bernoulli Bandit. The cumulative policy regret obtained by these methods averaged for $10$ simulations is shown in Fig.~\ref{fig:rotting}. Both the Bernoulli Bandit and all the GRU baselines fail to solve this task. Our method converges much faster than SWA and FEWA.  In contrast with SWA and FEWA, our method assumes nothing about the reward distribution, and therefore has to continuously explore. SWA and FEWA were designed to solve tasks with decaying rewards, and therefore do not require further exploration after a certain point. This is the reason our method keeps on accumulating regret, unlike SWA and FEWA that converge to the exact solution. 
	
	\noindent{\bf Sensitivity to the regularization parameter $\alpha_{EC}$} is examined   
	in Fig.~\ref{fig:senswheel1} for the Wheel Bandit and in Fig~\ref{fig:senswheel2} for the Rotating Wheel Bandit. 
	The performance is stable in both cases. 

	\section{Conclusions}
	
	The rewards obtained for a specific action are very often time dependent. In many cases, they are also dependent on the previous actions that were taken. In this work, we propose to employ an RNN in order to solve the non-stationary bandit problem. Our solution is general enough to model the contextual bandit, and is robust enough to address a diverse set of contextless non-stationary tasks, as well as stationary cases. 
	
	We note that training without proper regularization results in an overconfident learner which is incapable of exploring efficiently. Such an exploration is especially crucial when the rewards vary over time. We, therefore, suggest a new regularization term that minimizes the Boltzmann energy. This term leads to a bounded gap between the maximal and minimal probability assigned to each arm. Our experiments show that our method addresses multiple non-stationary rewards that vary according to time and also by action. 
	
	\paragraph*{Acknowledgments}
	This project was supported by a grant from the Tel Aviv University Center for AI and Data Science (TAD). The contribution of the first author is part of a Ph.D. thesis research conducted at Tel Aviv University.
	\bibliographystyle{IEEEbib}
	\bibliography{mab}
	
\end{document}